\documentclass[10pt,twocolumn,letterpaper]{article}
\usepackage{cvpr}
\usepackage{amsthm}
\usepackage{multirow}

\newcommand{\bd}{\boldsymbol}

\newtheorem{theorem}{Theorem}
\theoremstyle{definition}

\definecolor{cvprblue}{rgb}{0.21,0.49,0.74}
\usepackage[breaklinks,colorlinks,allcolors=cvprblue]{hyperref}

\title{Which Tokens Matter? Adaptive Token Selection for RLVR with the Relative Surprisal Index}

\author{
\textbf{Outongyi Lv}\textsuperscript{1,2,3*}\quad
\textbf{Yanzhao Zheng}\textsuperscript{2*}\quad
\textbf{Yuanwei Zhang}\textsuperscript{3*}\quad
\textbf{Zhenghao Huang}\textsuperscript{1,2*}\\
\textbf{Xingjun Wang}\textsuperscript{1,2}\quad
\textbf{Baohua Dong}\textsuperscript{2}\quad
\textbf{Hangcheng Zhu}\textsuperscript{2}\quad
\textbf{Yingda Chen}\textsuperscript{1,2 $\dag$ }\\
\textsuperscript{1}ModelScope Team, \textsuperscript{2}Alibaba Group, Hangzhou, China. \\
\textsuperscript{3}Shanghai Jiao Tong University, Shanghai, China. \\
}

\begin{document}
\maketitle
\begingroup
\renewcommand{\thefootnote}{\fnsymbol{footnote}}
\footnotetext[1]{These authors contributed equally to this work.}
\footnotetext[2]{Corresponding author.}
\endgroup

\begin{abstract}
Reinforcement learning (RL) has become a powerful tool for propelling Large Language Models (LLMs) beyond imitation-based training towards more robust reasoning capabilities. Among existing approaches, RL with Verifiable Rewards (RLVR) has emerged as a pivotal paradigm for advancing LLM reasoning. Despite its empirical success, recent studies have offered different insights. One line of inquiry advocates prioritizing high-entropy token positions during training, while another perspective cautions against allowing low-probability tokens to dominate gradient updates. Notably, although high-entropy tokens are usually correlated with low probability, both paradigms empirically yield substantial performance gains. In this work, we argue that evaluating sampled-token probability or entropy in isolation is insufficient to capture the policy optimization dynamics. To resolve this tension, we introduce the Relative Surprisal Index (RSI), a principled, information-theoretic metric that naturally couples the token's entropy with the probability of the selected token. We show that, under mild conditions, RSI is related to the local ratio between the first-order variations of the logit-gradient norm and predictive entropy under a selected-logit perturbation. Building on RSI, we propose RSI Selection (RSI-S), an entropy-adaptive token filtering method that retains tokens within a stable RSI interval. RSI-S successfully reconciles previous contradictory paradigms and filters out both redundant low-surprisal tokens and unstable high-surprisal tail tokens. Empirical evaluations show that RSI-S achieves higher avg@32 accuracy across different model scales (Qwen2.5-1.5B, 3B, and 7B) on AIME and AMC benchmarks: RSI-S improves avg@32 accuracy by 2--3 percentage points over GRPO. Overall, RSI offers a promising perspective for RLVR improvement.
\end{abstract}

\section{Introduction}
Large Language Models (LLMs) have achieved remarkable breakthroughs across diverse domains, including conversational systems \citep{achiam2023gpt, ouyang2022training, zeng2026glm}, generative modeling \citep{bie2026llada2, radford2018improving, team2026qwen3}, and mathematical reasoning \citep{yang2024qwen2,shao2025deepseekmath,wangself}. Within the post-training pipeline, Reinforcement Learning (RL) has emerged as a pivotal paradigm for overcoming the performance bottlenecks inherent in Supervised Fine-Tuning (SFT) \citep{chusft}. In particular, RL with Verifiable Rewards (RLVR) has gained substantial traction due to its practical utility and computational efficiency. A prominent baseline of this direction is Group Relative Policy Optimization (GRPO) \citep{shao2024deepseekmath}, a method introduced by DeepSeek that dispenses with the traditional value function by estimating advantages via rule-based rewards within each group. Building upon GRPO, several variants such as Dr. GRPO \citep{liuunderstanding}, Dynamic sAmpling Policy Optimization (DAPO) \citep{yu2026dapo}, and Group Sequence Policy Optimization (GSPO) \citep{zheng2025group} have been proposed to address its inherent limitations.

However, the above variants remain agnostic to the varying importance of individual tokens: they assign the same contribution to all tokens within a sequence. Recent work has begun to investigate token-level heterogeneity in RLVR, and existing methods can be broadly grouped into two categories. One line of work identifies subsets of influential tokens that affect the optimization dynamics \citep{wang2026beyond,lvgmts,mengsparse}. Another line of work adaptively adjusts token-level weights during training, effectively reweighting the contribution of different tokens \citep{ma2026fipo,yao2026,yang2025not}. Together, these studies highlight the importance of token-level signals in RLVR.

In this work, we study which tokens are beneficial for RLVR optimization. Our motivation comes from an apparent tension between two recent findings: on the one hand, \citet{wang2026beyond} advocate training exclusively on high-entropy token positions. Empirically, we find high predictive entropy tends to correlate with low probability assigned to the sampled token, although entropy and sampled-token probability are not equivalent. On the other hand, \citet{yang2025not} argue that low-probability tokens should not dominate LLM updates, showing theoretically that they can induce unstable gradient magnitudes. Despite this apparent tension, both approaches report empirical improvements on RLVR-based reasoning tasks. These observations suggest that neither sampled-token probability nor predictive entropy alone is sufficient to characterize a token's contribution to RLVR optimization. This naturally raises the following question:
\begin{center}
 \textit{How can we define a principled metric that unifies token entropy and sampled-token probability for identifying influential training tokens?}
\end{center}
\begin{figure*}[t]
  \centering
    \includegraphics[width=0.99\linewidth]{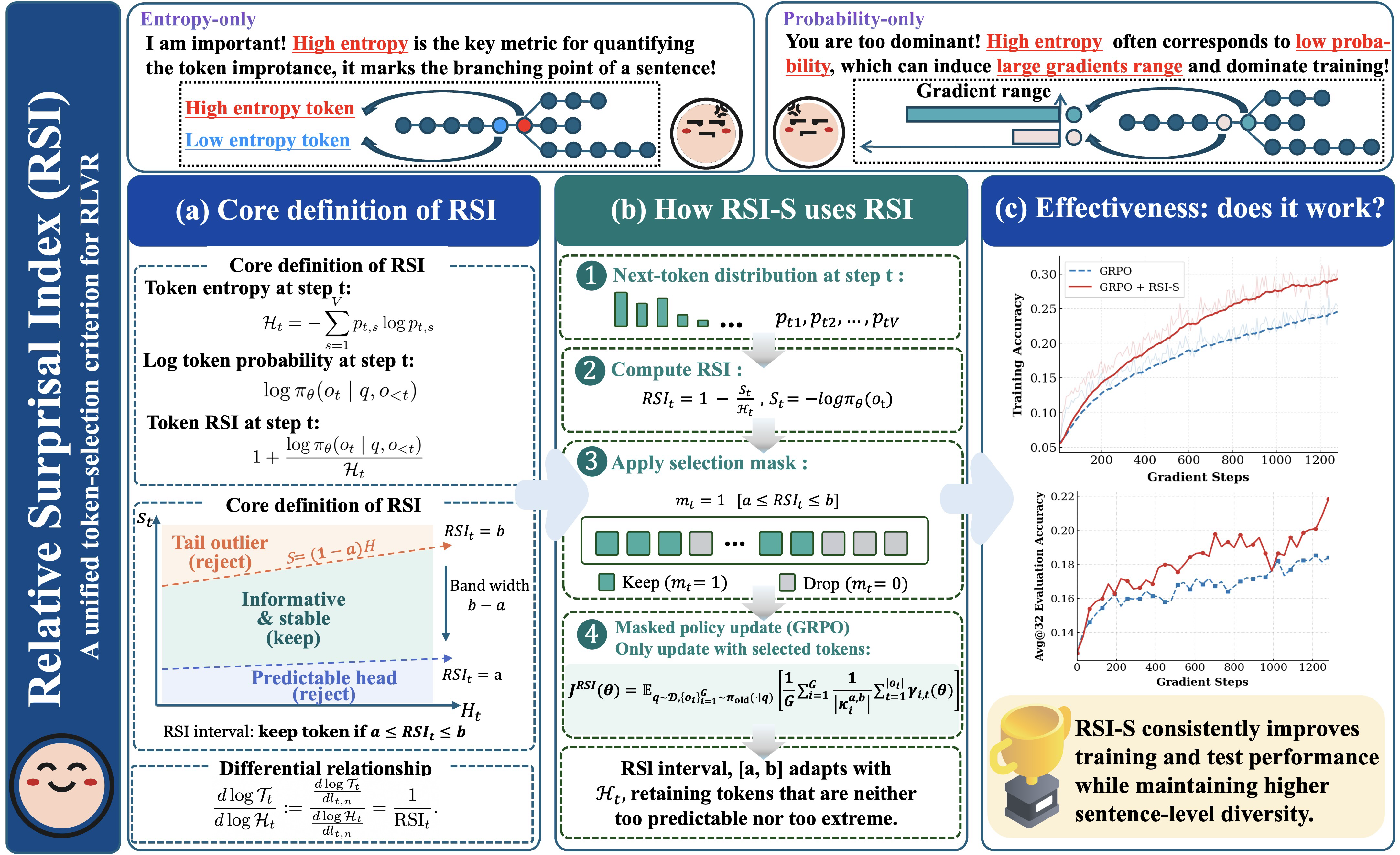} 
    \caption{Overall pipeline of the RSI framework.  (1) We collect all token-position entropy and sampled-token probability to compute the Relative Surprisal Index (RSI) by coupling these two quantities. (2) RSI-S introduces an interval-based token filtering mechanism within $[a, b]$ for RL training. Crucially, the lower bound $a$ removes excessively high-surprisal tail tokens, while the upper bound $b$ removes overly predictable low-surprisal tokens. (3) RSI-S enhances training stability and consistently outperforms other baselines across various model scales and benchmarks.}
    \label{fig:fig_overall}
\end{figure*}
To address this question, we introduce the Relative Surprisal Index (RSI), a principled token-level metric that inherently couples the token's entropy with its sampled-token probability. Compared with entropy alone, RSI further incorporates the likelihood of the realized token, allowing it to distinguish token positions with similar uncertainty but different realized-token probabilities. We show, both analytically and empirically, that RSI offers three main advantages as a token-selection criterion:
\begin{itemize}
    \item \textbf{Analytical connection to token-level gradients.}
    Under mild assumptions, we derive RSI from the relationship between token-level gradient magnitude and its entropy, providing a principled basis for token selection rather than a purely heuristic criterion.
    
    \item \textbf{Connecting entropy- and probability-based views.}
    RSI helps explain the apparent tension between \citet{wang2026beyond} and \citet{yang2025not}: it favors uncertain token positions while accounting for the probability assigned to the selected token. In this way, RSI integrates entropy- and probability-based criteria into a token-level score.

    \item \textbf{Consistent gains across model scales.}
    Experiments on AIME and AMC show that using RSI as a token-importance metric consistently outperforms the entropy-based 80/20 rule~\citep{wang2026beyond} and the probability-based method of~\citet{yang2025not}. Across Qwen2.5-1.5B, 3B, and 7B, RSI-S improves avg@32 accuracy by 2–3 percentage points over GRPO while producing shorter responses. 
\end{itemize}

\subsection{Related Work}
\noindent\textbf{GRPO and its variants.}
Group Relative Policy Optimization (GRPO) \citep{shao2024deepseekmath} motivates a series of recent variants that improve its efficiency, stability, and robustness. DAPO \citep{yu2026dapo} improves the training efficiency of GRPO by removing the KL regularizer and introducing dynamic sampling mechanisms. GSPO \citep{zheng2025group} revisits the importance-sampling formulation in GRPO and moves the correction from the token level to the sequence level. BAPO \citep{xi2025bapo} introduces an adaptive clipping rule to stabilize policy updates and balance the optimization process. GDPO \citep{liu2026gdpo} extends GRPO to multi-reward settings and is designed to mitigate reward hacking under multiple reward signals.

\noindent\textbf{Token selection in RLVR.}
\citet{wang2026beyond} identify an empirical ``80/20 rule,'' showing that a small fraction of high-entropy token positions can account for a large portion of the training benefit. Building on this observation, \citet{lvgmts} connect token entropy with gradient magnitude and propose prioritizing tokens with larger gradient signals. Other studies examine alternative token-level indicators: \citet{huan2025does} and \citet{chen2025reshaping} analyze token-level KL divergence and logit-rank shifts, while \citet{mengsparse} observe that RLVR tends to update a sparse subset of influential tokens. Collectively, these works indicate that effective RLVR does not require treating all tokens equally. However, existing token-selection criteria typically rely on a single signal, such as entropy, gradient magnitude, KL divergence, or logit-rank shift, leaving open how to jointly characterize predictive uncertainty and the probability of the selected token.

\noindent\textbf{Recalibrating token contributions.}
Another line of work modulates token contributions through reweighting rather than hard selection. \citet{yang2025not} highlight the risk of unstable gradients when low-probability tokens dominate LLM updates and propose adjusting token-level advantages to mitigate this issue. \citet{chen2025seed}, \citet{zhang2025edge}, \citet{wang2026entropy} and \citet{tang2026rethinking} study how the sign of the advantage function induces different update behaviors, motivating heterogeneous weighting strategies for positive and negative feedback. \citet{ma2026fipo} and \citet{yao2026} introduce future KL as a signal for estimating a token's downstream effect on subsequent generation, and use it to recalibrate token importance. \citet{zhu2026surprising} decouple the processing of positive and negative samples, introducing an adaptive reweighting mechanism specifically tailored for negative instances.

\section{Preliminaries}
The auto-regressive generation process of an LLM can be formally cast as a finite-horizon Markov Decision Process (MDP). Formally, the policy $\pi_{\theta}$ of the LLM is parameterized by weights $\theta$. Given an input query $\bd{q}$, the LLM generates a response $\bd{o} = (o_1,o_2,o_3,...,o_T)$ auto-regressively token by token, where $T$ denotes the sequence length. Specifically, at each step $t$, the next token $o_t$ is sequentially sampled from the conditional probability distribution governed by the preceding context $\pi_\theta(\cdot \mid \bd{q},\bd{o}_{<t}) = [p_{t,1}, \dots, p_{t,V}]$, where $\bd{o}_{<t}$ represents the previously generated tokens and $V$ is the vocabulary size.
\subsection{Group Relative Policy Optimization}
To mitigate the computational overhead of traditional proximal policy optimization (PPO) \citep{schulman2017proximal}, \citet{shao2024deepseekmath} introduce Group Relative Policy Optimization (GRPO). A defining characteristic of GRPO is that it dispenses with the auxiliary value network typically mandated by PPO. Instead, it leverages the statistical baseline of a localized group of sampled responses to estimate the token-level advantage. Specifically, for a given query $\bd{q}\sim \mathcal{D}$ with ground truth, GRPO uses the old policy $\pi_{\text{old}}$ to generate $G$ responses $\{\bd{o}_i\}_{i=1}^G$, and then the ground truth can be used to verify each response and get the reward list $R = [R_1,R_2,R_3,...,R_G]$, where $R_i=1$ if $\bd{o}_i$ matches the ground-truth answer and $R_i=0$ otherwise. The advantage for the $i$-th response is defined as:
$$
A_i = \frac{R_i - \text{mean}(R)}{\text{std}(R)}.
$$
To constrain policy updates and prevent the learned policy from deviating excessively from the reference policy $\pi_{\text{ref}}$, GRPO introduces a KL-divergence regularization term with respect to $\pi_{\text{ref}}$. Formally, the objective function of GRPO is formulated as follows:
$$
J^{GRPO}(\theta)=\mathbb{E}_{\bd{q}\sim \mathcal{D}, \{\bd{o}_i\}_{i=1}^G\sim \pi_{\text{old}}(\cdot|\bd{q})} \left[ \frac{1}{G}\sum_{i=1}^G\frac{1}{|\bd{o}_i|}\sum_{t=1}^{|\bd{o}_i|} \eta_{i, t}(\theta)\right].
$$
Here, $\eta_{i,t}(\theta)$ incorporates the standard PPO clipping mechanism alongside a reference policy KL penalty:
\begin{align*}
    \eta_{i,t}(\theta) & = \min(r_{i,t}(\theta)A_i,\text{clip}(r_{i,t}(\theta),1-\epsilon_1, 1+\epsilon_2)A_i) \\&- \beta \mathbb{D}_{\text{KL}}\big(\pi_{\theta}(\cdot \mid \bd{q}, \bd{o}_{i,<t}) \parallel \pi_{\text{ref}}(\cdot \mid \bd{q}, \bd{o}_{i,<t})).
\end{align*}

Here $\beta$ is the KL penalty coefficient, $\epsilon_1$ and $\epsilon_2$ are the clipping boundaries, which have been shown to be relevant to the exploration and exploitation of the model \citep{yu2026dapo}, $r_{i,t}(\theta)=\frac{\pi_{\theta}(o_{i,t} \mid \bd{q}, \bd{o}_{i, <t})}{\pi_{\text{old}}(o_{i,t} \mid \bd{q}, \bd{o}_{i, <t})}$.


\begin{figure*}[t]
  \centering
    \includegraphics[width=0.99\linewidth]{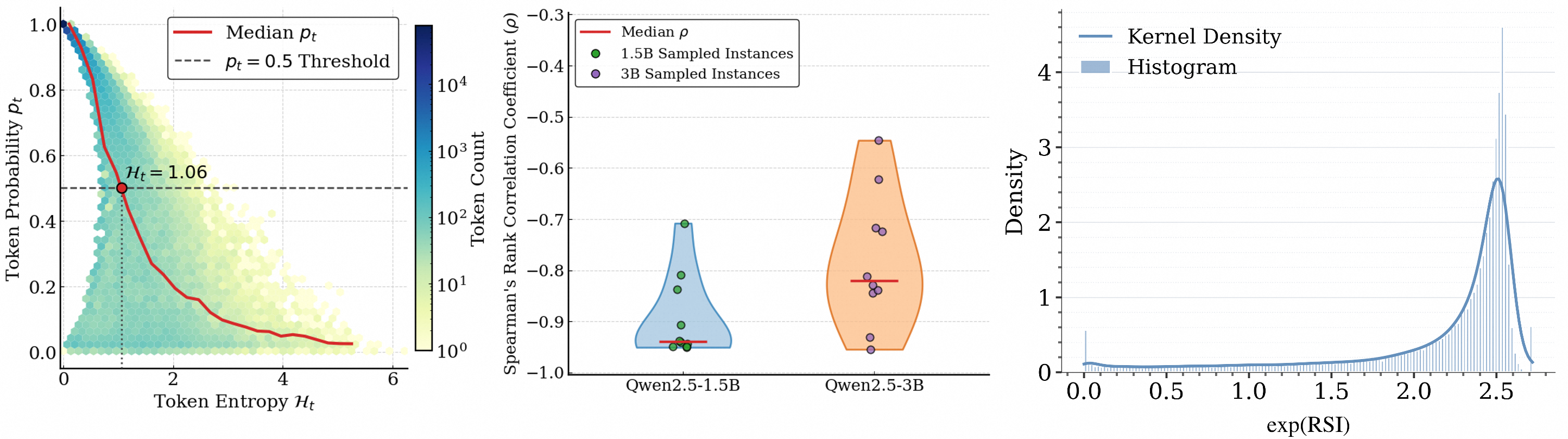} 
    \caption{\textbf{(Left)} Token entropy versus sampled-token probability on \textbf{Qwen2.5-1.5B}, showing that most tokens cluster in the high probability regime, with the median value decaying monotonically as entropy increases. \textbf{(Middle)} Spearman correlation coefficients between token entropy and probability across individual questions for \textbf{Qwen2.5-1.5B / 3B}, substantiating the stable inverse relationship where high entropy often signifies low probability. \textbf{(Right)} The empirical distribution of RSI, demonstrating a heavy concentration near $1$, which implies that the vast majority of generated tokens yield low surprisal.}
    \label{fig:fig1}
\end{figure*}

\section{Methodology}
Our method is motivated by a trade-off between exploration and optimization stability in current RLVR frameworks. On one hand, \citet{wang2026beyond} demonstrate that high-entropy tokens are important for effective learning, and recommend a training regime that prioritizes the top 20\% highest-entropy tokens for performance improvement. In LLMs, such high-entropy positions are often associated with low-probability tokens, as shown in Fig.~\ref{fig:fig1} (left and middle). On the other hand, \citet{yang2025not} caution that over-emphasizing low-probability tokens may destabilize the optimization process and introduce substantial gradient variance. Although seemingly different, the empirical success of both approaches suggests a deeper underlying connection. We posit that these methods are not contradictory, but instead reflect complementary aspects of how entropy and probability characterize the optimization dynamics in RLVR. Therefore, our goal is to define a token-level measure that combines both sources of information, rather than relying on entropy or sampled-token probability alone. Based on this observation, we find that entropy provides a natural normalization structure for token-level uncertainty. Specifically, for the token $o_t \sim \pi_\theta(\cdot \mid \bd{q}, \bd{o}_{<t})
= [p_{t,1}, \dots, p_{t,V}]$, the entropy $\mathcal{H}_t$ of this token is defined as:
$$
\mathcal{H}_t = -\sum_{s=1}^{V} p_{t,s}\log p_{t,s} = -\mathbb{E}_{p_{t,s}}[\log p_{t,s}],
$$
The entropy $\mathcal{H}_t$ can be interpreted as the negative expected log-probability under the token distribution $\pi_\theta(\cdot \mid \bd{q},\bd{o}_{<t})$. For a sampled token $o_t \sim \pi_\theta(\cdot \mid \bd{q}, \bd{o}_{<t})$, let $\pi_\theta(o_t \mid \bd{q}, \bd{o}_{<t}) = p_{t,n}$, its log-probability induces a deviation from the distribution-level average:
\begin{equation*}
\delta_{t}= \log(p_{t,n}) - \mathbb{E}_{p_{t,s}}[\log p_{t,s}],
\end{equation*}
$\delta_t$ measures the deviation of the selected token's log-probability from the distribution-level average log-probability. We then normalize $\delta_t$ using $\mathcal{H}_t$, yielding the Relative Surprisal Index (RSI):
\begin{equation}
\label{eq:RSI}
\mathrm{RSI}_t
=
\frac{\delta_t}{\mathcal{H}_t}
=
1+\frac{\log \pi_\theta(o_t\mid q,o_{<t})}{\mathcal{H}_t}.
\end{equation}
$\mathrm{RSI}_t$ is defined for non-degenerate distributions with $\mathcal{H}_t>0$. For degenerate distributions where $\mathcal{H}_t=0$, we define RSI by continuity as: $\lim_{\mathcal{H}_t\rightarrow 0}\mathrm{RSI}_t=1$, with the proof in Appendix 3. RSI measures the relative surprisal of the selected token with respect to the underlying token distribution. It is neither purely entropy-based nor purely probability-based. Specifically, $\mathrm{RSI}_t > 0$ indicates that the selected token has a higher log-probability than the distribution-level average, corresponding to lower-than-average surprisal. Conversely, $\mathrm{RSI}_t < 0$ indicates that the selected token is less likely than average, corresponding to higher surprisal. 

RSI quantifies the surprisal of the selected token relative to the underlying generative distribution. Beyond its intuitive interpretation, RSI admits a principled theoretical characterization. We show that under mild regularity conditions, RSI can be derived as a normalized form of the differential relationship between token-level entropy and the gradient of the log-probability at the last layer. For a given query $\bd{q}$ and its generated response $\bd{o}$, at each step $t$, let $\bd{l}_t = [l_{t,1}, l_{t,2}, \dots, l_{t,V}]$ denote the logits over the vocabulary produced by the final layer of the LLM. The next token $o_t$ is sampled from the distribution induced by the softmax transformation: $o_t \sim \pi_{\theta}(\cdot\mid \bd{q}, \bd{o}_{<t}) = \operatorname{Softmax}(\bd{l}_t):= [p_{t,1}, p_{t,2}, \dots, p_{t,V}]$. Then we have the following theorem:

\begin{theorem}
\label{theorem1}
Let $\mathcal{T}_{t}$ denote the $\ell_2$-norm of the gradient of the log-probability $\log \pi_{\theta}(o_t \mid \bd{q}, \bd{o}_{<t})$ with respect to the logits $\bd{l}_t$, and let $\mathcal{H}_t$ denote the entropy of the distribution
$\pi_{\theta}(\cdot \mid \bd{q}, \bd{o}_{<t})$. Suppose the generated token $o_t$ corresponds to the $n$-th vocabulary entry, i.e., $\pi_{\theta}(o_t \mid \bd{q}, \bd{o}_{<t}) = p_{t,n}$. Consider a small perturbation $l_{t,n} \rightarrow l_{t,n} + \Delta l_{t,n}$ and $\mathcal{H}_t>0$, $\mathrm{RSI}_t\neq 0 $. Then both $\log \mathcal{T}_t$ and $\log \mathcal{H}_t$ admit first-order variations with respect to $l_{t,n}$, we have:
$$
\frac{d\log \mathcal{T}_{t}}{d\log \mathcal{H}_t} := \frac{\frac{d \log \mathcal{T}_t}{d l_{t,n}}}
{\frac{d \log \mathcal{H}_t}{d l_{t,n}}} = \frac{1}{\mathrm{RSI}_t}.
$$
\end{theorem}
The proof of Theorem~\ref{theorem1} is provided in Appendix~\ref{appendix3}. Theorem~\ref{theorem1} shows that RSI characterizes the local sensitivity between token entropy and the logit-gradient norm. Thus, RSI provides a principled characterization of how perturbations to the logits associated with a selected token affect both predictive uncertainty and gradient magnitude. Under identical logit perturbations, a smaller magnitude $|\mathrm{RSI}_t|$ indicates a stronger coupling between changes in token entropy and the corresponding change in the logit-gradient norm. RSI values are bounded within $(-\infty, 1]$ and exhibit a strong concentration around $1$ (as illustrated in Figure~\ref{fig:fig1} (right)). This empirical distribution suggests that a vast majority of generated tokens carry limited informativeness, contributing marginally to gradient optimization—a phenomenon that aligns with the sparse training findings of \citet{mengsparse}. Furthermore, the empirical density is asymmetric around zero, reflecting an imbalance between the positive and negative RSI regimes. Inspired by Theorem~\ref{theorem1} and the intrinsic property that the expected RSI over the policy distribution $\pi_\theta(\cdot \mid \bd{q}, \bd{o}_{<t})$ is $0$, tokens with moderate RSI values may have more influence on gradient updates compared to extreme values, which motivates an interval-based token filtering anchored around the zero-mean RSI reference. To accommodate the asymmetric empirical density, we apply an asymmetric interval $[a,b]$ ($a < 0 < b$) where the negative tail threshold requires a larger magnitude ($|a| > b$). Formally, a token $o_t$ is retained for policy updates only if its $\mathrm{RSI}_t$ satisfies:
$$
\mathrm{RSI}_t \in [a,b].
$$
This mechanism selects tokens whose log-probabilities are aligned with an entropy-induced reference scale, while excluding both highly predictable tokens and excessively low-probability tail tokens. In other words, this criterion defines a bounded selection region over the current policy distribution, formalized as:
$$
e^{(a-1)\mathcal{H}_t} \leq \pi_\theta(o_t \mid q, o_{<t}) \leq e^{(b-1)\mathcal{H}_t}.
$$
By adjusting the upper bound $b$, the framework suppresses tokens with low surprisal, which correspond to high-probability and low-uncertainty regions of the model distribution. Concurrently, it preserves a moderate surprisal regime, preventing gradient updates from being dominated by rare high-surprisal tokens. On the other hand, adjusting the lower bound $a$ prevents instability caused by excessively high-surprisal tokens. At the same time, it retains a bounded subset of informative high-surprisal tokens that remain useful for policy updates. Overall, adjusting the dual thresholds $a$ and $b$ enables a unified control over token selection, connecting the high-entropy emphasis of \citet{wang2026beyond} with the stability considerations of \citet{yang2025not} and reducing updates from less-informative tokens while mitigating instability induced by extreme high-surprisal tokens.

Crucially, we emphasize that this filtration mechanism differs from standard static probability-thresholding approaches. Instead, RSI-based filtering induces an entropy-adaptive gating mechanism, where the allowed probability range of each token is modulated by the entropy $\mathcal{H}_t$. To operationalize this principle, we introduce a token importance calibration scheme, termed RSI Selection (\textbf{RSI-S}), and incorporate it into the optimization objective. Formally, the objective of RSI-S under GRPO is given by:
$$
J^{RSI}(\theta) = \mathbb{E}_{\bd{q}\sim \mathcal{D}, \{\bd{o}_i\}_{i=1}^G\sim \pi_{\text{old}}(\cdot|\bd{q})} \left[ \frac{1}{G}\sum_{i=1}^G\frac{1}{|\bd{\kappa}^{a,b}_i|}\sum_{t=1}^{|\bd{o}_i|} \gamma_{i,t}(\theta)\right],
$$
where
$$
\begin{aligned}
\gamma_{i,t}(\theta) &= {\mathbb{I}[\mathrm{RSI}_{i, t}\in [a,b]]}\cdot \eta_{i, t}(\theta)\\
&=\begin{cases} 
\eta_{i, t}(\theta), & \text{if }\  \mathrm{RSI}_{i, t} \in [a,b], \\
0, & \text{otherwise.}
\end{cases}
\end{aligned}
$$
and $|\boldsymbol{\kappa}^{a,b}_i| = \#\{t\mid \mathrm{RSI}_{i, t} \in [a,b]\}$ denotes the cardinality of the subset of tokens within the output sequence $\boldsymbol{o}_i$ that actively contribute to the gradient updates. To avoid the ill-defined normalization in our objective, if $|\boldsymbol{\kappa}^{a,b}_i| = 0$, it means the corresponding trajectory $\bd{o}_i$ is excluded from the update, i.e., it contributes zero to the objective function.

\section{Experiments}
\subsection{Training and Evaluation Settings} Our distributed training pipeline is built on \textbf{EasyR1}\footnote{\url{https://github.com/hiyouga/EasyR1/tree/main}}, a modular RLVR framework implemented on top of \textbf{verl}\footnote{\url{https://github.com/verl-project/verl}}. To avoid potential confounding effects from supervised fine-tuning (SFT) on downstream mathematical corpora \citep{wu2026reasoning}, we conduct all experiments using base language models only, specifically \textbf{Qwen2.5-1.5B}, \textbf{Qwen2.5-3B}, and \textbf{Qwen2.5-7B}. We utilize the \textbf{DAPO-MATH-17K} dataset \citep{yu2026dapo} as the primary training set and adopt the following unified prompt template across all training stages:

\begin{figure*}[t]
  \centering
    \includegraphics[width=0.97\linewidth]{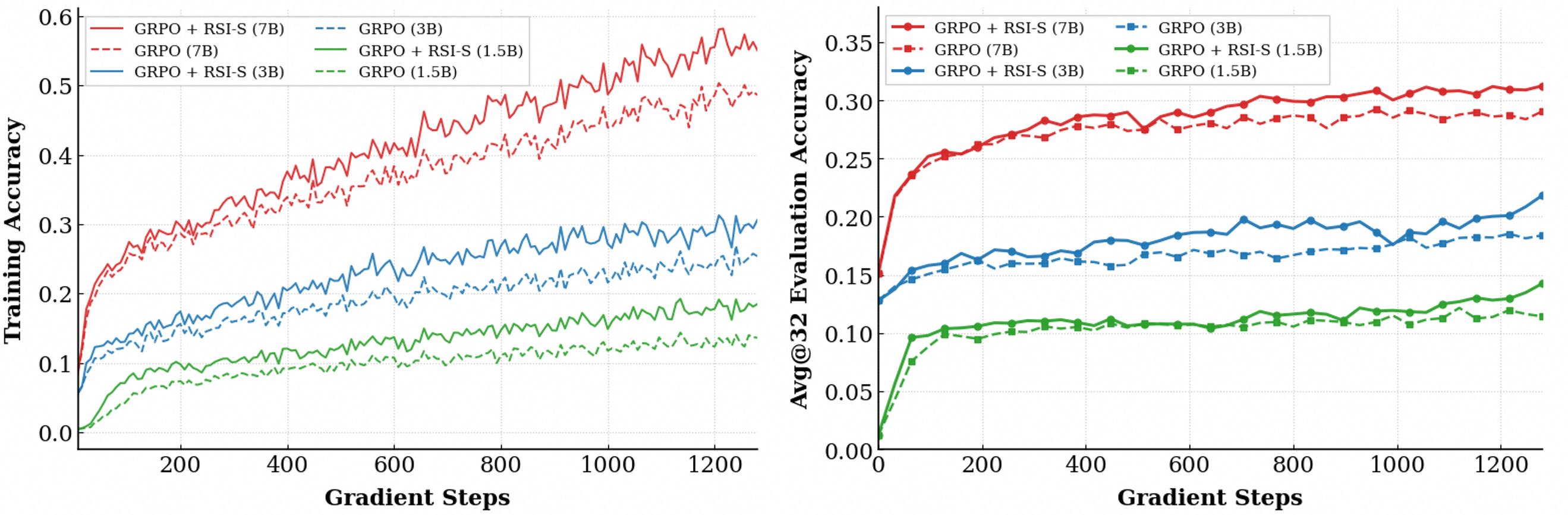} 
    \caption{We contrast baseline GRPO and GRPO with RSI-S across different model scales. Training accuracy curves over gradient steps \textbf{(Left)} and evaluation-results curves (each under avg@32 accuracy) across 6 challenging mathematical benchmarks \textbf{(Right)}.}
    \label{fig:fig3}
\end{figure*}

\begin{center}
\fbox{
\begin{minipage}{0.95\linewidth}
\small\texttt{Please solve this problem step by step, and put your final answer within $\backslash$boxed\{\}.}
\end{minipage}
}
\end{center}

\noindent To evaluate the reasoning and generalization capabilities of the trained policies, we report performance on two benchmark families covering 6 contest-year subsets: \textbf{AIME} (\textbf{AIME 2024}$\sim$\textbf{AIME 2026}) and \textbf{AMC} (\textbf{AMC 2022}$\sim$\textbf{AMC 2024}). We adopt sample-averaged accuracy under a 32-sample rollout protocol (denoted as \textbf{avg@32}) as the primary evaluation metric. Formally, for each evaluation prompt, the model generates \textbf{32} independent response trajectories, and the score is computed as the mean accuracy over these completions. We provide additional details of the training and evaluation procedures in Appendix~\ref{appendix1}.

\subsection{Main Results}
We adopt the 80/20 rule proposed by \citet{wang2026beyond} as our \textbf{entropy baseline} (EB)\footnote{\url{https://github.com/Shenzhi-Wang/Beyond-the-80-20-Rule-RLVR/tree/main}}. We follow the recommendation and retain the top 20\% highest-entropy tokens to compute the GRPO loss. For the \textbf{probability baseline} (PB)\footnote{\url{https://github.com/zhyang2226/AR-Lopti}}, we employ the method introduced by \citet{yang2025not}. For our RSI-S framework, the interval hyperparameters are set to $[-6, 0.95]$, $[-6, 0.9]$, and $[-6, 0.8]$ for the 1.5B, 3B, and 7B models, respectively. Importantly, this scale-specific configuration is not an artifact of meticulous hyperparameter tuning. Rather, it is motivated by the observation that models exhibit distinct behavioral patterns at varying scales, necessitating a downward shift of the surprisal upper bound as capacity increases. Nevertheless, this variation does not imply that a uniform interval degrades RSI-S performance, a point we elaborate on in the subsequent section. The empirical results are summarized in Tables~\ref{tab:main_results_1.5b}--\ref{tab:main_results_7b} and Figure~\ref{fig:fig3}. We evaluate the avg@32 accuracy and average response lengths of Qwen2.5 (1.5B, 3B, and 7B) under the GRPO, EB, PB, and RSI-S frameworks. Crucially, RSI-S consistently enhances performance across all model scales, surpassing both the EB and PB baselines. Moreover, RSI-S yields shorter average response lengths. Furthermore, we observe that under GRPO, the performance of EB and PB is markedly less pronounced on smaller models, exposing their inherent limitations. This phenomenon underscores that RSI-S is a unified framework that effectively synthesizes the strengths of both entropy and probability.

\begin{table*}[h!]
\centering
\caption{The experimental results of GRPO, GRPO + EB, GRPO + PB, GRPO + RSI-S across \textbf{6} math reasoning benchmarks on \textbf{Qwen2.5-1.5B}. We report \textbf{avg@32} accuracy (Acc, \%) and average generation length (Len), along with the improvements.}
\label{tab:main_results_1.5b}
\begin{tabular}{lcccccc|cccc}
\toprule
\multirow{2}{*}{\textbf{Benchmark}} & \multicolumn{2}{c}{\textbf{GRPO}} & \multicolumn{2}{c}{\textbf{GRPO + EB}} & \multicolumn{2}{c}{\textbf{GRPO + PB}} & \multicolumn{2}{c}{\textbf{GRPO + RSI-S}} & \multicolumn{2}{c}{\textbf{Gain vs. GRPO}} \\
\cmidrule(lr){2-3} \cmidrule(lr){4-5} \cmidrule(lr){6-7} \cmidrule(lr){8-9} \cmidrule(lr){10-11}
 & Acc & Len & Acc & Len & Acc & Len & Acc & Len & $\Delta$ Acc & $\Delta$ Len \\
\midrule
AIME (2024$\sim$2026) & 2.36 &  1220.65 & 2.22 & 1138.05 & 2.12 & 1084.09 & \textbf{2.47} & 1082.70 & \textbf{+0.11} & -137.95 \\
AMC (2022$\sim$2024) & 21.95 &  997.41 & 21.06 & 994.19 & 21.54 & 913.13 & \textbf{26.04} & 917.96 & \textbf{+4.09} & -79.45 \\
\midrule
\textbf{Average} & 12.15 &  1109.03 & 11.64 & 1066.12 & 11.83 & 998.61 & \textbf{14.25} & 1000.33 & \textbf{+2.10} & -108.70 \\
\bottomrule
\end{tabular}
\end{table*}

\begin{table*}[h!]
\centering
\caption{The experimental results of GRPO, GRPO + EB, GRPO + PB, GRPO + RSI-S across \textbf{6} math reasoning benchmarks on \textbf{Qwen2.5-3B}. We report \textbf{avg@32} accuracy (Acc, \%) and average generation length (Len), along with the improvements.}
\label{tab:main_results_3b}
\begin{tabular}{lcccccc|cccc}
\toprule
\multirow{2}{*}{\textbf{Benchmark}} & \multicolumn{2}{c}{\textbf{GRPO}} & \multicolumn{2}{c}{\textbf{GRPO + EB}} & \multicolumn{2}{c}{\textbf{GRPO + PB}} & \multicolumn{2}{c}{\textbf{GRPO + RSI-S}} & \multicolumn{2}{c}{\textbf{Gain vs. GRPO}} \\
\cmidrule(lr){2-3} \cmidrule(lr){4-5} \cmidrule(lr){6-7} \cmidrule(lr){8-9} \cmidrule(lr){10-11}
 & Acc & Len & Acc & Len & Acc & Len & Acc & Len & $\Delta$ Acc & $\Delta$ Len \\
\midrule
AIME (2024$\sim$2026) & 5.07 &  1311.26 & 3.85 & 1208.87 & 3.61 & 991.41 & \textbf{5.73} & 1117.29 & \textbf{+0.66} & -193.97 \\
AMC (2022$\sim$2024) & 31.98 &  1098.70 & 31.88 & 1030.31 & 33.02 & 854.23 & \textbf{37.93} & 926.96 & \textbf{+5.95} & -171.74 \\
\midrule
\textbf{Average} & 18.53 &  1204.98 & 17.86 & 1119.59 & 18.32 & 922.82 & \textbf{21.83} & 1022.12 & \textbf{+3.30} & -182.86 \\
\bottomrule
\end{tabular}
\end{table*}

\begin{table*}[h!]
\centering
\caption{The experimental results of GRPO, GRPO + EB, GRPO + PB, GRPO + RSI-S across \textbf{6} math reasoning benchmarks on \textbf{Qwen2.5-7B}. We report \textbf{avg@32} accuracy (Acc, \%) and average generation length (Len), along with the improvements.}
\label{tab:main_results_7b}
\begin{tabular}{lcccccc|cccc}
\toprule
\multirow{2}{*}{\textbf{Benchmark}} & \multicolumn{2}{c}{\textbf{GRPO}} & \multicolumn{2}{c}{\textbf{GRPO + EB}} & \multicolumn{2}{c}{\textbf{GRPO + PB}} & \multicolumn{2}{c}{\textbf{GRPO + RSI-S}} & \multicolumn{2}{c}{\textbf{Gain vs. GRPO}} \\
\cmidrule(lr){2-3} \cmidrule(lr){4-5} \cmidrule(lr){6-7} \cmidrule(lr){8-9} \cmidrule(lr){10-11}
 & Acc & Len & Acc & Len & Acc & Len & Acc & Len & $\Delta$ Acc & $\Delta$ Len \\
\midrule
AIME (2024$\sim$2026) & 10.69 &  1314.80 & 9.97 & 1277.61 & 10.56 & 1124.84 & \textbf{11.25} & 969.56 & \textbf{+0.56} & -345.24 \\
AMC (2022$\sim$2024) & 47.41 &  1019.43 & 49.79 & 1036.50 & 47.19 & 908.22 & \textbf{51.24} & 832.76 & \textbf{+3.83} & -186.67 \\
\midrule
\textbf{Average} & 29.05 &  1167.11 & 29.88 & 1157.06 & 28.87 & 1016.53 & \textbf{31.24} & 901.16 & \textbf{+2.19} & -265.95 \\
\bottomrule
\end{tabular}
\end{table*}

\begin{figure*}[t]
  \centering
  \includegraphics[width=0.99\linewidth]{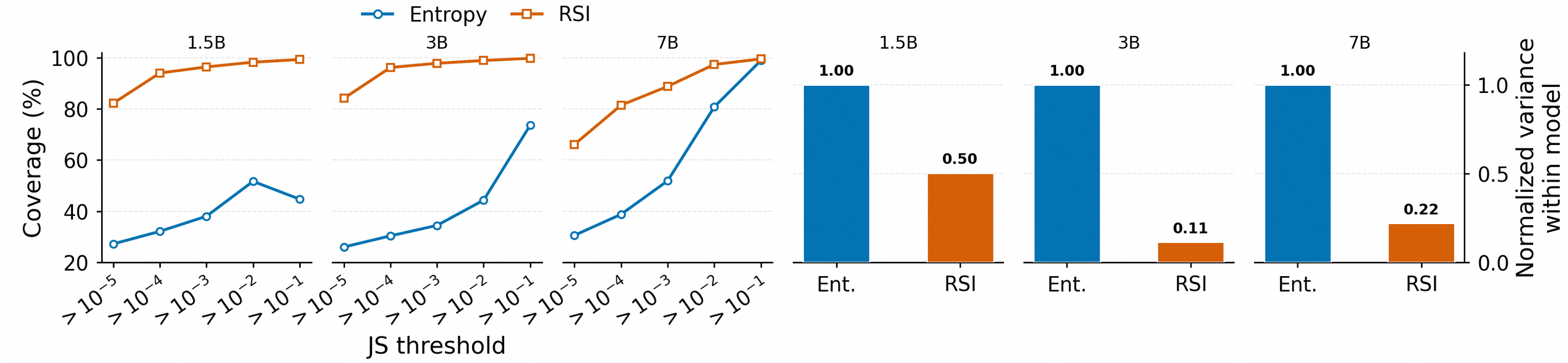} 
    \caption{Recall/coverage of high-JS tokens under different JS thresholds for 1.5B, 3B, and 7B models, comparing RSI with high-entropy token selection. RSI consistently achieves substantially higher coverage than entropy-based selection across model scales and JS thresholds (left). Threshold sensitivity is measured by the per-model normalized variance of coverage across JS thresholds. RSI exhibits only 0.50×, 0.11×, and 0.22× of the sensitivity of entropy for 1.5B, 3B, and 7B, respectively, indicating that RSI more stably captures high-JS tokens under threshold shifts (right).}
    \label{fig:fig4}
\end{figure*}

\subsection{Ablation Study and Discussion}
\textbf{RSI-S vs. Static Probability Thresholds.}  To verify that RSI-S is not equivalent to simple probability truncation, we evaluate \textbf{Qwen2.5-3B} using fixed probability intervals. Tokens outside these predefined ranges (\textbf{P1}: [0.1, 0.9], \textbf{P2}: [0.2, 0.8], \textbf{P3}: [0.3, 0.7]) are strictly masked to neutralize their gradient contributions. As shown in Table~\ref{tab:ab1} (extended in Appendix~\ref{appendix2}), these static baselines fail to achieve the consistent optimization gains of \text{RSI-S}. This performance gap confirms that rather than applying a rigid heuristic, \text{RSI-S} dynamically scales the valid probability space based on token-level entropy, thereby preserving critical tokens that static intervals inadvertently discard.

\begin{table}[t]
\centering
\caption{The final avg@32 accuracy (Acc, \%) and average generation length (Len) performance comparison of GRPO, GRPO + RSI-S and GRPO + P1, P2 and P3 across 6 math reasoning benchmarks on \textbf{Qwen2.5-3B}.}
\label{tab:ab1}
\resizebox{\columnwidth}{!}{
\begin{tabular}{lcccc}
\toprule
\textbf{Method} & \textbf{Acc (\%)} & \textbf{Len} & \textbf{Imp. Acc (\%)} & \textbf{Imp. Len} \\ \midrule
GRPO             & 18.53 & 1204.98 & -- & -- \\
GRPO + RSI-S     & \textbf{21.83} & 1022.12 & \textbf{+3.30} & -182.86 \\
GRPO + P1        & 16.16 & 1058.22 & -2.37 & -146.76\\
GRPO + P2        & 15.84 & 1041.63 & -2.69 & -163.35\\
GRPO + P3        & 15.99 & 1075.54 & -2.54 & -129.44\\
\bottomrule
\end{tabular}
}
\end{table}

\noindent \textbf{The Necessity of Balanced RSI Intervals.} 
To verify the necessity of both RSI bounds, we perform a factorial ablation study along two axes. First, we examine the effect of removing the upper bound by comparing $[a, b]$ with $[a, 1]$ (RSI-SU), which tests whether suppressing overly predictable tokens is necessary. Second, we examine the effect of removing the lower bound by comparing $[a, b]$ with $[-\infty, b]$ (RSI-SL), which tests whether excluding high-surprisal tail tokens is beneficial. As shown in Table~\ref{tab:ab2} for \textbf{Qwen2.5-3B} (extended in Appendix~\ref{appendix2}), both one-sided variants consistently underperform the balanced interval $[a, b]$. This highlights that effective token selection requires a delicate trade-off. Specifically, RSI bounds unify the complementary hypotheses of \citet{wang2026beyond} and \citet{yang2025not}: upper bound $b$ removes predictable low-surprisal tokens; lower bound $a$ removes extreme high-surprisal tail tokens. Removing either bound disrupts this balance.

\begin{table}[t]
\centering
\caption{The final avg@32 accuracy (Acc, \%) and average generation length (Len) performance comparison of GRPO, GRPO + RSI-S, GRPO + RSI-SU and GRPO + RSI-SL across 6 math reasoning benchmarks on \textbf{Qwen2.5-3B}.}
\label{tab:ab2}
\resizebox{\columnwidth}{!}{
\begin{tabular}{lcccc}
\toprule
\textbf{Method} & \textbf{Acc (\%)} & \textbf{Len} & \textbf{Imp. Acc (\%)} & \textbf{Imp. Len} \\ \midrule
GRPO             & 18.53 & 1204.98 & -- & -- \\
GRPO + RSI-S     & \textbf{21.83} & 1022.12 & \textbf{+3.30} & -182.86 \\
GRPO + RSI-SU    & 19.14 & 999.59  & +0.61 & -205.39 \\
GRPO + RSI-SL    & 18.45 & 1101.98 & -0.08 & -103.00 \\
\bottomrule
\end{tabular}
}
\end{table}

\noindent \textbf{Effect of Model Scale on RSI Upper Bound.} As demonstrated in Table~\ref{tab:ab3}, we find the best-performing upper bound $b$  among the tested values shifts toward lower values as model capacity increases, even though our method consistently maintains its superiority under identical settings. To further elucidate this, Figure~\ref{fig:fig4} (left) shows that 1.5b, 3b and 7b have significant differences in behavioral patterns. We attribute this discrepancy to the learning dynamics at varying scales: smaller models primarily prioritize exploitation and knowledge consolidation, thus demanding a larger proportion of low-surprisal tokens. Conversely, larger models exhibit a stronger capacity for active exploration; excessively predictable (low-surprisal) tokens become less critical for optimization, thereby driving the downward shift of the $b$-value. This perspective is well-corroborated by recent literature \citep{cui2025entropy, hao2025rethinking, lu2025generalization, chen2026does}.

\begin{table}[t]
\centering
\caption{The final avg@32 accuracy (Acc, \%) performance comparison of GRPO and GRPO + RSI-S under different RSI intervals across 6 math reasoning benchmarks and model sizes (1.5B, 3B and 7B).}
\label{tab:ab3}
\resizebox{\columnwidth}{!}{
\begin{tabular}{lcccc}
\toprule
\textbf{Model Size} & \textbf{GRPO} & \textbf{+ RSI-S[-6,0.8]} & \textbf{+ RSI-S[-6,0.9]} & \textbf{+ RSI-S[-6,0.95]} \\ \midrule
1.5B    & 12.15 & 12.55 (+0.40)  & 12.69 (+0.54) & \textbf{14.25 (+2.10)} \\
3.0B    & 18.53 & 19.45 (+0.92)  & \textbf{21.83 (+3.30)} & 19.31 (+0.78) \\
7.0B    & 29.05  & \textbf{31.24 (+2.19)}  & 29.13 (+0.08) & 30.11 (+1.06) \\
\bottomrule
\end{tabular}
}
\end{table}

\noindent \textbf{RSI Stably Captures RL-Critical Tokens.}
Following \citet{mengsparse}, we compute the per-token Jensen--Shannon (JS) divergence between pre-RL and post-RL output distributions on 1.5B, 3B, and 7B models, and regard high-JS tokens as those more affected by RL. Figure~\ref{fig:fig4} compares the recall of high-JS tokens by RSI and entropy-based selection under different JS thresholds. RSI consistently covers more high-JS tokens than entropy across model scales, while also exhibiting lower threshold sensitivity. Specifically, the normalized variance of RSI is only $0.50\times$, $0.11\times$, and $0.22\times$ that of entropy on 1.5B, 3B, and 7B models, respectively. This suggests that RSI more stably captures RL-critical tokens, aligning with \citet{mengsparse}'s observation that high-JS tokens are more critical after RL.

\section{Conclusion}
In this paper, we introduce the Relative Surprisal Index (RSI), a novel metric designed to quantify token importance. Moving beyond existing paradigms, RSI successfully reconciles the seemingly contradictory conclusions derived from predictive entropy and sampled-token probability, unifying both perspectives through RSI-S, a token-selection method built on RSI. Crucially, we demonstrate that under the conditions specified in Theorem 1, RSI emerges as a mathematically principled consequence of entropy and gradient differentiation, rather than a mere heuristic formulation. Finally, we conduct extensive numerical experiments across different model scales to evaluate the scalability and generalization of RSI-S. Coupled with rigorous ablation studies, the empirical results provide encouraging evidence that RSI delivers consistent, stable, and parameter-robust performance gains for RLVR. These findings underscore RSI's potential as a promising token-selection criterion for RLVR.

{
    \small
    \bibliographystyle{ieeenat_fullname}
    \bibliography{main}

\begin{thebibliography}{35}
\providecommand{\natexlab}[1]{#1}
\providecommand{\url}[1]{\texttt{#1}}
\expandafter\ifx\csname urlstyle\endcsname\relax
  \providecommand{\doi}[1]{doi: #1}\else
  \providecommand{\doi}{doi: \begingroup \urlstyle{rm}\Url}\fi

\bibitem[Achiam et~al.(2023)Achiam, Adler, Agarwal, Ahmad, Akkaya, Aleman, Almeida, Altenschmidt, Altman, Anadkat, et~al.]{achiam2023gpt}
Josh Achiam, Steven Adler, Sandhini Agarwal, Lama Ahmad, Ilge Akkaya, Florencia~Leoni Aleman, Diogo Almeida, Janko Altenschmidt, Sam Altman, Shyamal Anadkat, et~al.
\newblock Gpt-4 technical report.
\newblock \emph{arXiv preprint arXiv:2303.08774}, 2023.

\bibitem[Bie et~al.(2026)Bie, Cao, Cao, Chen, Chen, Chen, Du, Feng, Feng, Gong, et~al.]{bie2026llada2}
Tiwei Bie, Maosong Cao, Xiang Cao, Bingsen Chen, Fuyuan Chen, Kun Chen, Lun Du, Daozhuo Feng, Haibo Feng, Mingliang Gong, et~al.
\newblock Llada2. 1: Speeding up text diffusion via token editing.
\newblock \emph{arXiv preprint arXiv:2602.08676}, 2026.

\bibitem[Chen et~al.(2025{\natexlab{a}})Chen, Chen, Wang, and Yang]{chen2025seed}
Minghan Chen, Guikun Chen, Wenguan Wang, and Yi Yang.
\newblock Seed-grpo: Semantic entropy enhanced grpo for uncertainty-aware policy optimization.
\newblock \emph{arXiv preprint arXiv:2505.12346}, 2025{\natexlab{a}}.

\bibitem[Chen et~al.(2025{\natexlab{b}})Chen, Li, and Zou]{chen2025reshaping}
Xingwu Chen, Tianle Li, and Difan Zou.
\newblock Reshaping reasoning in llms: A theoretical analysis of rl training dynamics through pattern selection.
\newblock \emph{arXiv preprint arXiv:2506.04695}, 2025{\natexlab{b}}.

\bibitem[Chen et~al.(2026)Chen, Lu, Zhao, Wang, Yue, Song, and Huang]{chen2026does}
Zhiqi Chen, Rui Lu, Andrew Zhao, Zhaokai Wang, Yang Yue, Shiji Song, and Gao Huang.
\newblock Does reinforcement learning really incentivize reasoning capacity in llms beyond the base model?
\newblock \emph{Advances in Neural Information Processing Systems}, 38:\penalty0 57654--57689, 2026.

\bibitem[Chu et~al.(2025)Chu, Zhai, Yang, Tong, Xie, Schuurmans, Le, Levine, and Ma]{chusft}
Tianzhe Chu, Yuexiang Zhai, Jihan Yang, Shengbang Tong, Saining Xie, Dale Schuurmans, Quoc~V Le, Sergey Levine, and Yi Ma.
\newblock Sft memorizes, rl generalizes: A comparative study of foundation model post-training.
\newblock In \emph{Forty-second International Conference on Machine Learning}, 2025.

\bibitem[Cui et~al.(2025)Cui, Zhang, Chen, Yuan, Wang, Zuo, Li, Fan, Chen, Chen, et~al.]{cui2025entropy}
Ganqu Cui, Yuchen Zhang, Jiacheng Chen, Lifan Yuan, Zhi Wang, Yuxin Zuo, Haozhan Li, Yuchen Fan, Huayu Chen, Weize Chen, et~al.
\newblock The entropy mechanism of reinforcement learning for reasoning language models.
\newblock \emph{arXiv preprint arXiv:2505.22617}, 2025.

\bibitem[Hao et~al.(2025)Hao, Wang, Liu, Luo, Yu, Dong, Lin, Wang, and Chen]{hao2025rethinking}
Zhezheng Hao, Hong Wang, Haoyang Liu, Jian Luo, Jiarui Yu, Hande Dong, Qiang Lin, Can Wang, and Jiawei Chen.
\newblock Rethinking entropy interventions in rlvr: An entropy change perspective.
\newblock \emph{arXiv preprint arXiv:2510.10150}, 2025.

\bibitem[Huan et~al.(2025)Huan, Li, Zheng, Xu, Kim, Du, Poovendran, Neubig, and Yue]{huan2025does}
Maggie Huan, Yuetai Li, Tuney Zheng, Xiaoyu Xu, Seungone Kim, Minxin Du, Radha Poovendran, Graham Neubig, and Xiang Yue.
\newblock Does math reasoning improve general llm capabilities? understanding transferability of llm reasoning.
\newblock \emph{arXiv preprint arXiv:2507.00432}, 2025.

\bibitem[Liu et~al.(2026)Liu, Dong, Lu, Diao, Belcak, Liu, Chen, Yin, Wang, Cheng, et~al.]{liu2026gdpo}
Shih-Yang Liu, Xin Dong, Ximing Lu, Shizhe Diao, Peter Belcak, Mingjie Liu, Min-Hung Chen, Hongxu Yin, Yu-Chiang~Frank Wang, Kwang-Ting Cheng, et~al.
\newblock Gdpo: Group reward-decoupled normalization policy optimization for multi-reward rl optimization.
\newblock \emph{arXiv preprint arXiv:2601.05242}, 2026.

\bibitem[Liu et~al.(2025)Liu, Chen, Li, Qi, Pang, Du, Lee, and Lin]{liuunderstanding}
Zichen Liu, Changyu Chen, Wenjun Li, Penghui Qi, Tianyu Pang, Chao Du, Wee~Sun Lee, and Min Lin.
\newblock Understanding r1-zero-like training: A critical perspective.
\newblock In \emph{Second Conference on Language Modeling}, 2025.

\bibitem[Lu et~al.(2025)Lu, Zhao, Sun, Peng, Ding, and Mei]{lu2025generalization}
Brian Lu, Hongyu Zhao, Shuo Sun, Hao Peng, Rui Ding, and Hongyuan Mei.
\newblock Generalization of rlvr using causal reasoning as a testbed.
\newblock \emph{arXiv preprint arXiv:2512.20760}, 2025.

\bibitem[Lv et~al.(2026)Lv, Zhang, et~al.]{lvgmts}
Outongyi Lv, Yuanwei Zhang, et~al.
\newblock Gmts: Gradient magnitude-based token selection improves rlvr training for llm reasoning, 2026.

\bibitem[Ma et~al.(2026)Ma, Yang, Huang, Lu, Meng, Wang, Ding, Vosoughi, Wang, and Zhou]{ma2026fipo}
Chiyu Ma, Shuo Yang, Kexin Huang, Jinda Lu, Haoming Meng, Shangshang Wang, Bolin Ding, Soroush Vosoughi, Guoyin Wang, and Jingren Zhou.
\newblock Fipo: Eliciting deep reasoning with future-kl influenced policy optimization.
\newblock \emph{arXiv preprint arXiv:2603.19835}, 2026.

\bibitem[Meng et~al.(2026)Meng, Huang, Wei, Ma, Yang, Wang, Wang, Ding, and Zhou]{mengsparse}
Haoming Meng, Kexin Huang, Shaohang Wei, Chiyu Ma, Shuo Yang, Xue Wang, Guoyin Wang, Bolin Ding, and Jingren Zhou.
\newblock Sparse but critical: A token-level analysis of distributional shifts in rlvr fine-tuning of llms.
\newblock In \emph{The Fourteenth International Conference on Learning Representations}, 2026.

\bibitem[Ouyang et~al.(2022)Ouyang, Wu, Jiang, Almeida, Wainwright, Mishkin, Zhang, Agarwal, Slama, Ray, et~al.]{ouyang2022training}
Long Ouyang, Jeffrey Wu, Xu Jiang, Diogo Almeida, Carroll Wainwright, Pamela Mishkin, Chong Zhang, Sandhini Agarwal, Katarina Slama, Alex Ray, et~al.
\newblock Training language models to follow instructions with human feedback.
\newblock \emph{Advances in neural information processing systems}, 35:\penalty0 27730--27744, 2022.

\bibitem[Radford et~al.(2018)Radford, Narasimhan, Salimans, Sutskever, et~al.]{radford2018improving}
Alec Radford, Karthik Narasimhan, Tim Salimans, Ilya Sutskever, et~al.
\newblock Improving language understanding by generative pre-training.
\newblock 2018.

\bibitem[Schulman et~al.(2017)Schulman, Wolski, Dhariwal, Radford, and Klimov]{schulman2017proximal}
John Schulman, Filip Wolski, Prafulla Dhariwal, Alec Radford, and Oleg Klimov.
\newblock Proximal policy optimization algorithms.
\newblock \emph{arXiv preprint arXiv:1707.06347}, 2017.

\bibitem[Shao et~al.(2024)Shao, Wang, Zhu, Xu, Song, Bi, Zhang, Zhang, Li, et~al.]{shao2024deepseekmath}
Zhihong Shao, Peiyi Wang, Qihao Zhu, Runxin Xu, Junxiao Song, Xiao Bi, Haowei Zhang, Mingchuan Zhang, YK Li, et~al.
\newblock Deepseekmath: Pushing the limits of mathematical reasoning in open language models.
\newblock \emph{arXiv preprint arXiv:2402.03300}, 2024.

\bibitem[Shao et~al.(2025)Shao, Luo, Lu, Ren, Hu, Ye, Gou, Ma, and Zhang]{shao2025deepseekmath}
Zhihong Shao, Yuxiang Luo, Chengda Lu, ZZ Ren, Jiewen Hu, Tian Ye, Zhibin Gou, Shirong Ma, and Xiaokang Zhang.
\newblock Deepseekmath-v2: Towards self-verifiable mathematical reasoning.
\newblock \emph{arXiv preprint arXiv:2511.22570}, 2025.

\bibitem[Tang et~al.(2026)Tang, Zhan, Li, Zhao, Zhang, Wen, Zhang, and Zhou]{tang2026rethinking}
Xinyu Tang, Yuliang Zhan, Zhixun Li, Wayne~Xin Zhao, Zhenduo Zhang, Zujie Wen, Zhiqiang Zhang, and Jun Zhou.
\newblock Rethinking sample polarity in reinforcement learning with verifiable rewards.
\newblock In \emph{Proceedings of the 64th Annual Meeting of the Association for Computational Linguistics (Volume 1: Long Papers)}, pages 2928--2954, 2026.

\bibitem[Team(2026)]{team2026qwen3}
Qwen Team.
\newblock Qwen3. 5-omni technical report.
\newblock \emph{arXiv preprint arXiv:2604.15804}, 2026.

\bibitem[Wang et~al.(2026{\natexlab{a}})Wang, Xie, Zhang, Sun, Chen, Li, and Zhang]{wang2026entropy}
Shumin Wang, Yuexiang Xie, Wenhao Zhang, Yuchang Sun, Yanxi Chen, Yaliang Li, and Yanyong Zhang.
\newblock On the entropy dynamics in reinforcement fine-tuning of large language models.
\newblock \emph{arXiv preprint arXiv:2602.03392}, 2026{\natexlab{a}}.

\bibitem[Wang et~al.(2026{\natexlab{b}})Wang, Yu, Gao, Zheng, Liu, Lu, Dang, Chen, Yang, Zhang, et~al.]{wang2026beyond}
Shenzhi Wang, Le Yu, Chang Gao, Chujie Zheng, Shixuan Liu, Rui Lu, Kai Dang, Xiong-Hui Chen, Jianxin Yang, Zhenru Zhang, et~al.
\newblock Beyond the 80/20 rule: High-entropy minority tokens drive effective reinforcement learning for llm reasoning.
\newblock \emph{Advances in Neural Information Processing Systems}, 38:\penalty0 115452--115486, 2026{\natexlab{b}}.

\bibitem[Wang et~al.(2023)Wang, Wei, Schuurmans, Le, Chi, Narang, Chowdhery, and Zhou]{wangself}
Xuezhi Wang, Jason Wei, Dale Schuurmans, Quoc~V Le, Ed~H Chi, Sharan Narang, Aakanksha Chowdhery, and Denny Zhou.
\newblock Self-consistency improves chain of thought reasoning in language models.
\newblock In \emph{The Eleventh International Conference on Learning Representations}, 2023.

\bibitem[Wu et~al.(2026)Wu, Zhang, Dong, Xi, Zhao, Jin, Fan, Zhou, Lv, Zhang, et~al.]{wu2026reasoning}
Mingqi Wu, Zhihao Zhang, Qiaole Dong, Zhiheng Xi, Jun Zhao, Senjie Jin, Xiaoran Fan, Yuhao Zhou, Huijie Lv, Ming Zhang, et~al.
\newblock Reasoning or memorization? unreliable results of reinforcement learning due to data contamination.
\newblock In \emph{Proceedings of the AAAI Conference on Artificial Intelligence}, pages 33944--33952, 2026.

\bibitem[Xi et~al.(2025)Xi, Guo, Nan, Zhou, Shen, Chen, Liu, Huang, Zhang, Guo, et~al.]{xi2025bapo}
Zhiheng Xi, Xin Guo, Yang Nan, Enyu Zhou, Junrui Shen, Wenxiang Chen, Jiaqi Liu, Jixuan Huang, Zhihao Zhang, Honglin Guo, et~al.
\newblock Bapo: Stabilizing off-policy reinforcement learning for llms via balanced policy optimization with adaptive clipping.
\newblock \emph{arXiv preprint arXiv:2510.18927}, 2025.

\bibitem[Yang et~al.(2024)Yang, Zhang, Hui, Gao, Yu, Li, Liu, Tu, Zhou, Lin, et~al.]{yang2024qwen2}
An Yang, Beichen Zhang, Binyuan Hui, Bofei Gao, Bowen Yu, Chengpeng Li, Dayiheng Liu, Jianhong Tu, Jingren Zhou, Junyang Lin, et~al.
\newblock Qwen2. 5-math technical report: Toward mathematical expert model via self-improvement.
\newblock \emph{arXiv preprint arXiv:2409.12122}, 2024.

\bibitem[Yang et~al.(2025)Yang, Luo, Wang, Han, He, Li, and Xu]{yang2025not}
Zhihe Yang, Xufang Luo, Zilong Wang, Dongqi Han, Zhiyuan He, Dongsheng Li, and Yunjian Xu.
\newblock Do not let low-probability tokens over-dominate in rl for llms.
\newblock In \emph{2nd AI for Math Workshop@ ICML}, 2025.

\bibitem[Yao et~al.(2026)Yao, Wang, et~al.]{yao2026}
Jiarui Yao, Ruida Wang, et~al.
\newblock Future-kl regularized grpo: Process-level credit assignment from f-divergence regularization.
\newblock \emph{arXiv preprint arXiv:2601.10201}, 2026.

\bibitem[Yu et~al.(2026)Yu, Zhang, Zhu, Yuan, Zuo, Yue, Dai, Fan, Liu, Liu, et~al.]{yu2026dapo}
Qiying Yu, Zheng Zhang, Ruofei Zhu, Yufeng Yuan, Xiaochen Zuo, Yu Yue, Weinan Dai, Tiantian Fan, Gaohong Liu, Lingjun Liu, et~al.
\newblock Dapo: An open-source llm reinforcement learning system at scale.
\newblock \emph{Advances in Neural Information Processing Systems}, 38:\penalty0 113222--113244, 2026.

\bibitem[Zeng et~al.(2026)Zeng, Lv, Hou, Du, Zheng, Chen, Yin, Ge, Huang, Xie, et~al.]{zeng2026glm}
Aohan Zeng, Xin Lv, Zhenyu Hou, Zhengxiao Du, Qinkai Zheng, Bin Chen, Da Yin, Chendi Ge, Chenghua Huang, Chengxing Xie, et~al.
\newblock Glm-5: from vibe coding to agentic engineering.
\newblock \emph{arXiv preprint arXiv:2602.15763}, 2026.

\bibitem[Zhang et~al.(2025)Zhang, Wen, Wu, and Huang]{zhang2025edge}
Xingjian Zhang, Siwei Wen, Wenjun Wu, and Lei Huang.
\newblock Edge-grpo: Entropy-driven grpo with guided error correction for advantage diversity.
\newblock \emph{arXiv preprint arXiv:2507.21848}, 2025.

\bibitem[Zheng et~al.(2025)Zheng, Liu, Li, Chen, Yu, Gao, Dang, Liu, Men, Yang, et~al.]{zheng2025group}
Chujie Zheng, Shixuan Liu, Mingze Li, Xiong-Hui Chen, Bowen Yu, Chang Gao, Kai Dang, Yuqiong Liu, Rui Men, An Yang, et~al.
\newblock Group sequence policy optimization.
\newblock \emph{arXiv preprint arXiv:2507.18071}, 2025.

\bibitem[Zhu et~al.(2026)Zhu, Xia, Wei, Chen, Chen, and Meng]{zhu2026surprising}
Xinyu Zhu, Mengzhou Xia, Zhepei Wei, Wei-Lin Chen, Danqi Chen, and Yu Meng.
\newblock The surprising effectiveness of negative reinforcement in llm reasoning.
\newblock \emph{Advances in Neural Information Processing Systems}, 38:\penalty0 126546--126573, 2026.

\end{thebibliography}
}

\clearpage

\section{Appendix}
\subsection{Appendix A}
\label{appendix1}
All experiments are conducted with random seed fixed to $0$ unless otherwise specified. For the \textbf{PB} experiments, we directly follow the released implementation of \citet{yang2025not}. For the \textbf{EB} experiments, we port the key micro-batch processing routine from \citet{wang2026beyond} into our EasyR1-based training pipeline, which enables token-level filtering under the same training framework as our method. During training, responses are sampled with temperature $1.0$ and top-$p$ $1.0$. The maximum input length and maximum output length are set to $1024$ and $4096$, respectively. We use a learning rate of $1\times 10^{-6}$. The batch size and mini-batch size are set to $1024$ and $128$, respectively; therefore, each outer training step is divided into $8$ mini-batch gradient updates. The GRPO group size is set to $G=16$. We train each model for $10$ epochs. Since \texttt{drop\_last} is enabled for the training set, this corresponds to $160$ outer training steps and $1280$ mini-batch gradient updates in total. Unless explicitly stated otherwise, all remaining hyperparameters follow the default EasyR1 configuration. For evaluation, we use temperature $0.6$ and top-$p$ $0.95$. We perform avg@32 evaluation every $4$ outer training steps, where each evaluation example is sampled $32$ times and the reported score is averaged over these samples.

\subsection{Appendix B}
\label{appendix2}
We report the additional ablation results in Tables~\ref{tab:ab_appendix21}--\ref{tab:ab_appendix24}. These results complement the main ablations in Tables~\ref{tab:ab1} and~\ref{tab:ab2}. Together with Table~\ref{tab:ab1}, Tables~\ref{tab:ab_appendix21} and~\ref{tab:ab_appendix22} show that the improvement of RSI-S cannot be attributed to a fixed hard probability threshold. Instead, RSI-S defines an entropy-adaptive probability interval for token selection: the effective selection range changes according to the uncertainty of the policy distribution rather than remaining fixed across all decoding states. This pattern is consistent across different model scales, suggesting that the benefit of RSI-S comes from adaptive token selection rather than from a manually chosen static cutoff. Tables~\ref{tab:ab_appendix23} and~\ref{tab:ab_appendix24} further extend the analysis in Table~\ref{tab:ab2} by ablating the two RSI-S boundaries, $a$ and $b$. The results indicate that both constraints are necessary for stable and effective token selection. The lower bound removes extreme high-surprisal tail tokens and the upper bound removes predictable low-surprisal tokens. Removing either boundary weakens the stability of the selected token set and leads to inferior performance. These findings support the design choice of using both lower and upper RSI constraints in RSI-S.

\begin{table}[t]
\centering
\caption{The final avg@32 accuracy (Acc, \%) and average generation length (Len) performance comparison of GRPO, GRPO + RSI-S and GRPO + P1, P2 and P3 across 6 math reasoning benchmarks on \textbf{Qwen2.5-1.5B}.}
\label{tab:ab_appendix21}
\resizebox{\columnwidth}{!}{
\begin{tabular}{lcccc}
\toprule
\textbf{Method} & \textbf{Acc (\%)} & \textbf{Len} & \textbf{Imp. Acc (\%)} & \textbf{Imp. Len} \\ \midrule
GRPO             & 12.15 & 1109.03 & -- & -- \\
GRPO + RSI-S     & \textbf{14.25} & 1000.33 & \textbf{+2.10} & -108.70 \\
GRPO + P1        & 11.35 & 1064.06 & -0.80 & -44.97 \\
GRPO + P2        & 11.44 & 949.45  & -0.71 & -159.58 \\
GRPO + P3        & 11.05 & 1017.59 & -1.10 & -91.44 \\
\bottomrule
\end{tabular}
}
\end{table}

\begin{table}[t]
\centering
\caption{The final avg@32 accuracy (Acc, \%) and average generation length (Len) performance comparison of GRPO, GRPO + RSI-S and GRPO + P1, P2 and P3 across 6 math reasoning benchmarks on \textbf{Qwen2.5-7B}.}
\label{tab:ab_appendix22}
\resizebox{\columnwidth}{!}{
\begin{tabular}{lcccc}
\toprule
\textbf{Method} & \textbf{Acc (\%)} & \textbf{Len} & \textbf{Imp. Acc (\%)} & \textbf{Imp. Len} \\ \midrule
GRPO             & 29.05 & 1167.11 & -- & -- \\
GRPO + RSI-S     & \textbf{31.24} & 901.16 & \textbf{+2.19} & -265.95 \\
GRPO + P1        & 26.14 & 1115.01 & -2.91 & -52.10 \\
GRPO + P2        & 24.16 & 947.22 & -4.89 & -219.89 \\
GRPO + P3        & 24.56 & 992.48 & -4.49 & -174.63 \\
\bottomrule
\end{tabular}
}
\end{table}

\begin{table}[t]
\centering
\caption{The final avg@32 accuracy (Acc, \%) and average generation length (Len) performance comparison of GRPO, GRPO + RSI-S, GRPO + RSI-SU and GRPO + RSI-SL across 6 math reasoning benchmarks on \textbf{Qwen2.5-1.5B}.}
\label{tab:ab_appendix23}
\resizebox{\columnwidth}{!}{
\begin{tabular}{lcccc}
\toprule
\textbf{Method} & \textbf{Acc (\%)} & \textbf{Len} & \textbf{Imp. Acc (\%)} & \textbf{Imp. Len} \\ \midrule
GRPO             & 12.15 & 1109.03 & -- & -- \\
GRPO + RSI-S     & \textbf{14.25} & 1000.33 & \textbf{+2.10} & -108.70 \\
GRPO + RSI-SU    & 12.64 & 1089.25  & +0.49 & -19.78 \\
GRPO + RSI-SL    & 11.81 & 1118.04  & -0.34 & +9.01 \\
\bottomrule
\end{tabular}
}
\end{table}

\begin{table}[t]
\centering
\caption{The final avg@32 accuracy (Acc, \%) and average generation length (Len) performance comparison of GRPO, GRPO + RSI-S, GRPO + RSI-SU and GRPO + RSI-SL across 6 math reasoning benchmarks on \textbf{Qwen2.5-7B}.}
\label{tab:ab_appendix24}
\resizebox{\columnwidth}{!}{
\begin{tabular}{lcccc}
\toprule
\textbf{Method} & \textbf{Acc (\%)} & \textbf{Len} & \textbf{Imp. Acc (\%)} & \textbf{Imp. Len} \\ \midrule
GRPO             & 29.05 & 1167.11 & -- & -- \\
GRPO + RSI-S     & \textbf{31.24} & 901.16 & \textbf{+2.19} & -265.95 \\
GRPO + RSI-SU    & 28.93 & 1000.66  & -0.12 & -166.45 \\
GRPO + RSI-SL    & 29.17 & 1016.53  & +0.12 & -150.58 \\
\bottomrule
\end{tabular}
}
\end{table}

\subsection{Appendix C}
\label{appendix3}
\noindent \textbf{Proposition.}
Assume that the vocabulary size $V$ is finite. Let 
$p_{t,n}=\pi_\theta(o_t\mid q,o_{<t})$. If $p_{t,n}\to1$, or equivalently
$p_{t,n}=1-\varepsilon$ with $\varepsilon\to0$, then
$$
\lim_{\mathcal{H}_t\rightarrow 0}\mathrm{RSI}_t=1.
$$
\begin{proof}
By definition,
$$
\mathrm{RSI}_t=1+\frac{\log \pi_\theta(o_t\mid q,o_{<t})}{\mathcal{H}_t}.
$$
According to the definition of $\mathcal{H}_t$:
\begin{align*}
    \mathcal{H}_t  &= -\sum_{s=1}^{V} p_{t,s}\log p_{t,s} \\
    &= - p_{t,n}\log p_{t,n} - \\&\sum_{s\neq n}(1-p_{t,n})\frac{p_{t,s}}{(1-p_{t,n})} \log ((1-p_{t,n})\frac{p_{t,s}}{(1-p_{t,n})}) 
    \\&=  - p_{t,n}\log p_{t,n}-(1-p_{t,n}) \log(1-p_{t,n}) -\\&
    (1-p_{t,n})\sum_{s\neq n} \frac{p_{t,s}}{(1-p_{t,n})}\log \frac{p_{t,s}}{(1-p_{t,n})}.
    \\
\end{align*}
Observe that $-\sum_{s\neq n}\frac{p_{t,s}}{(1-p_{t,n})}\log \frac{p_{t,s}}{(1-p_{t,n})}$ is the entropy of the normalized residual distribution over the remaining
vocabulary. We denote it by $\mathcal{H}_t(p/p_{t,n})$. Moreover, $- p_{t,n}\log p_{t,n}-(1-p_{t,n}) \log(1-p_{t,n})$ is the binary entropy of the split $[p_{t,n},1-p_{t,n}]$, denoted by $\mathcal{H}_t(p_{t,n})$. Hence,
$$
\mathcal{H}_t = \mathcal{H}_t(p_{t,n}) + (1-p_{t,n}) \mathcal{H}_t(p/p_{t,n}).
$$
Therefore,
\begin{align*}
    \lim_{\mathcal{H}_t\rightarrow 0}\mathrm{RSI}_t & = \lim_{\varepsilon \rightarrow 0}\mathrm{RSI}_t = 1+\lim_{\varepsilon \rightarrow 0} \frac{\log(1-\varepsilon)}{\mathcal{H}_t}.
\end{align*}
Where:
$$
\mathcal{H}_t= -\varepsilon\log\varepsilon-(1-\varepsilon)\log((1-\varepsilon))+\varepsilon \mathcal{H}_t(p/p_{t,n}),
$$
It suffices to show that:
$$
\lim_{\varepsilon\rightarrow 0}\frac{\log(1-\varepsilon)}{-\varepsilon\log\varepsilon-(1-\varepsilon)\log((1-\varepsilon))+\varepsilon \mathcal{H}_t(p/p_{t,n})}=0.
$$
Let $D_{\varepsilon} = -\varepsilon\log\varepsilon-(1-\varepsilon)\log((1-\varepsilon))+\varepsilon \mathcal{H}_t(p/p_{t,n})$. It remains to prove that:
$$
\lim_{\varepsilon\rightarrow 0}\frac{\log(1-\varepsilon)}{D_{\varepsilon}}=0.
$$
Notice that:
$$
D_{\varepsilon} \geq -\varepsilon\log\varepsilon.
$$
Then we have:
$$
0 \leq |\frac{\log(1-\varepsilon)}{D_{\varepsilon}}|\leq \frac{-\log(1-\varepsilon)}{-\varepsilon\log\varepsilon}.
$$
So according to Taylor expansion:
$$
\lim_{\varepsilon\rightarrow 0}\frac{-\log(1-\varepsilon)}{-\varepsilon\log\varepsilon} = \lim_{\varepsilon\rightarrow 0} \frac{\varepsilon+O(\varepsilon^2)}{\varepsilon\log(1/\varepsilon)}=\lim_{\varepsilon\rightarrow 0} \frac{1+O(\varepsilon)}{\log(1/\varepsilon)} = 0.
$$
Thus, by the squeeze theorem,
\[
\left|
\frac{\log(1-\varepsilon)}{D_\varepsilon}
\right|
\to0.
\]
Therefore,
\[
\frac{\log(1-\varepsilon)}{D_\varepsilon}\to0.
\]
\end{proof}
\textbf{Theorem 1:} Let $\mathcal{T}_{t}$ denote the $\ell_2$-norm of the gradient of the log-probability $\log \pi_{\theta}(o_t \mid \bd{q}, \bd{o}_{<t})$ with respect to the logits $\bd{l}_t$, and let $\mathcal{H}_t$ denote the entropy of the distribution
$\pi_{\theta}(\cdot \mid \bd{q}, \bd{o}_{<t})$. Suppose the generated token $o_t$ corresponds to the $n$-th vocabulary entry, i.e., $\pi_{\theta}(o_t \mid \bd{q}, \bd{o}_{<t}) = p_{t,n}$. Consider a small perturbation $l_{t,n} \rightarrow l_{t,n} + \Delta l_{t,n}$ and $\mathcal{H}_t>0$, $\mathrm{RSI}_t\neq 0 $. Then both $\log \mathcal{T}_t$ and $\log \mathcal{H}_t$ admit first-order variations with respect to $l_{t,n}$, we have:
$$
\frac{d\log \mathcal{T}_{t}}{d\log \mathcal{H}_t} := \frac{\frac{d \log \mathcal{T}_t}{d l_{t,n}}}
{\frac{d \log \mathcal{H}_t}{d l_{t,n}}} = \frac{1}{\mathrm{RSI}_t}.
$$

\begin{proof}
First, by the softmax definition, for any $i$,
$$
p_{t,i} = \frac{e^{l_{t,i}}}{\sum_{j=1}^V{e^{l_{t,j}}}},
$$
Therefore,
$$\frac{\partial p_{t,i}}{\partial l_{t,j}} = -p_{t,j}p_{t,i} (j\neq i),  \frac{\partial p_{t,i}}{\partial l_{t,i}} = p_{t,i}(1-p_{t,i}).
$$
According to $\pi_{\theta}(\cdot|\bd{q}, \bd{o}_{<t}) = [p_{t,1},p_{t,2},...,p_{t,V}]$, we have:
\begin{align*}
\nabla_{\bd{l}_t} \log(p_{t,i}) = \frac{\nabla_{\bd{l}_t} p_{t,i}}{p_{t,i}} & =\frac{1}{p_{t,i}}[p_{t,i}(e_{t,i}-\pi_{\theta}(\cdot|\bd{q}, \bd{o}_{<t}))] \\& = e_{t,i}-\pi_{\theta}(\cdot|\bd{q}, \bd{o}_{<t}).
\end{align*}
where $e_{t,i}$ is the unit vector at step $t$ whose i-th entry is 1 and the others are 0. 

\noindent According to the definition of $\mathcal{T}_{t}$:
$$
\mathcal{T}_{t} = ||\nabla_{\bd{l}_t} \log(p_{t,n})||_2 = \sqrt{(1-p_{t,n})^2 + \sum_{j \neq n }p_{t,j}^2}.
$$
Consider the total differentials of $\mathcal{H}_{t}$ and $\mathcal{T}_{t}$ with respect to $\bd{l}_t$, consider perturbing $\bd{l}_t$ by $\bd{l}_t +\Delta \bd{l}_t$ where $\Delta \bd{l}_t = [\Delta l_{t,1}, \Delta l_{t,2}, \dots, \Delta l_{t,V}]$. Then we have:
$$
d(\log \mathcal{T}_{t}) = \nabla_{\bd{l}_t} \log \mathcal{T}_{t} \cdot \Delta \bd{l}_t = \sum_{k=1}^{V} \frac{\partial \log \mathcal{T}_{t}}{\partial l_{t,k}} \Delta l_{t,k},
$$
$$
d(\log \mathcal{H}_{t}) = \nabla_{\bd{l}_t} \log \mathcal{H}_{t} \cdot \Delta \bd{l}_t = \sum_{k=1}^{V} \frac{\partial \log \mathcal{H}_{t}}{\partial l_{t,k}} \Delta l_{t,k}.
$$
We first compute $\frac{\partial \mathcal{H}_{t}}{\partial l_{t,i}}$ first, we have:
$$
\frac{\partial \mathcal{H}_{t}}{\partial l_{t,i}} = \sum_k \frac{\partial \mathcal{H}_{t}}{\partial p_{t,k}} \cdot \frac{\partial p_{t,k}}{\partial l_{t,i}},
$$
and
\begin{align*}
\frac{\partial \mathcal{H}_{t}}{\partial p_{t,k}} = \frac{\partial (-p_{t,k} \log p_{t,k})}{\partial p_{t,k}} & = -(\log p_{t,k} + p_{t,k} \cdot \frac{1}{p_{t,k}}) \\ &= -(\log p_{t,k} + 1).
\end{align*}
Therefore:
\begin{align*}
&\frac{\partial \mathcal{H}_{t}}{\partial l_{t,i}}  = \sum_k \left[ \frac{\partial \mathcal{H}_{t}}{\partial p_{t,k}}  \cdot \frac{\partial p_{t,k}}{\partial l_{t,i}} \right]  \\ &= \sum_k \left[ -(\log p_{t,k} + 1) \cdot p_{t,k}(\delta^t_{ki} - p_{t,i}) \right] \\ & = \sum_k \left[ -p_{t,k} \log p_{t,k} (\delta^t_{ki} - p_{t,i}) \right] + \sum_k \left[ -p_{t,k} (\delta^t_{ki} - p_{t,i}) \right].
\end{align*}
Where:
$$
\delta^t_{ki}=
\begin{cases} 
1, & \text{if } i=k , \\
0, & \text{if } i \neq k.
\end{cases}
$$
Notice that:
$$
\sum_k -p_{t,k} \delta^t_{ki} + \sum_k p_{t,k} p_{t,i} = -p_{t,i} + p_{t,i} (\sum_k p_{t,k}) = 0.
$$
So:
\begin{align*}
\frac{\partial \mathcal{H}_{t}}{\partial l_{t,i}} & = \sum_k -p_{t,k} \log p_{t,k} \delta^t_{ki} - \sum_k (-p_{t,k} \log p_{t,k}) p_{t,i} \\ &= -p_{t,i} (\log p_{t,i} + \mathcal{H}_{t}) \\ & = -p_{t,i}\mathcal{H}_{t}(1+\frac{\log(p_{t,i})}{\mathcal{H}_{t}}).
\end{align*}
and
\begin{align*}
\sum_{k=1}^{V} \frac{\partial \mathcal{H}_{t}}{\partial l_{t,k}} \Delta l_{t,k}= -\sum_{k=1}^{V} p_{t,k} (\log p_{t,k} + \mathcal{H}_{t})\Delta l_{t,k}.
\end{align*}
On the other hand, we compute the derivative of $\mathcal T_t$ with respect
to each logit $l_{t,i}$. Since
\[
\mathcal T_t^2
=
\|\mathbf e_n-\mathbf p_t\|_2^2
=
(1-p_{t,n})^2+\sum_{k\neq n}p_{t,k}^2
=
1-2p_{t,n}+\sum_{k=1}^V p_{t,k}^2,
\]
we have
\[
\frac{\partial \mathcal T_t^2}{\partial l_{t,i}}
=
-2\frac{\partial p_{t,n}}{\partial l_{t,i}}
+
2\sum_{k=1}^V
p_{t,k}
\frac{\partial p_{t,k}}{\partial l_{t,i}}.
\]
Using the softmax derivative
\[
\frac{\partial p_{t,k}}{\partial l_{t,i}}
=
p_{t,k}(\delta_{ki}^t-p_{t,i}),
\]
we obtain
\begin{align*}
\frac{\partial \mathcal T_t^2}{\partial l_{t,i}}
&=
-2p_{t,n}(\delta_{ni}^t-p_{t,i})
+
2\sum_{k=1}^V p_{t,k}^2(\delta_{ki}^t-p_{t,i}) \\
&=
-2p_{t,n}\delta_{ni}^t
+
2p_{t,n}p_{t,i}
+
2p_{t,i}^2
-
2p_{t,i}\sum_{k=1}^V p_{t,k}^2.
\end{align*}
Since $p_{t,n}\delta_{ni}^t=p_{t,i}\delta_{ni}^t$, this can be rewritten as
\[
\frac{\partial \mathcal T_t^2}{\partial l_{t,i}}
=
-2p_{t,i}
\left(
\delta_{ni}^t
-
p_{t,n}
-
p_{t,i}
+
\sum_{k=1}^V p_{t,k}^2
\right).
\]
Moreover,
\[
\frac{\partial \mathcal T_t^2}{\partial l_{t,i}}
=
2\mathcal T_t
\frac{\partial \mathcal T_t}{\partial l_{t,i}}.
\]
Therefore,
\[
\frac{\partial \mathcal T_t}{\partial l_{t,i}}
=
-\frac{p_{t,i}}{\mathcal T_t}
\left(
\delta_{ni}^t
-
p_{t,n}
-
p_{t,i}
+
\sum_{k=1}^V p_{t,k}^2
\right).
\]
Using
\[
\sum_{k=1}^V p_{t,k}^2
=
\mathcal T_t^2+2p_{t,n}-1,
\]
we further obtain
\[
\frac{\partial \mathcal T_t}{\partial l_{t,i}}
=
-\frac{p_{t,i}}{\mathcal T_t}
\left(
\delta_{ni}^t
-
p_{t,i}
+
\mathcal T_t^2
+
p_{t,n}
-
1
\right).
\]
Equivalently,
\[
\frac{\partial \mathcal T_t}{\partial l_{t,i}}
=
\begin{cases}
-\mathcal T_t p_{t,n},
& i=n,\\[4pt]
-\mathcal T_t p_{t,i}
+
\dfrac{p_{t,i}^2-p_{t,i}p_{t,n}+p_{t,i}}{\mathcal T_t},
& i\neq n.
\end{cases}
\]
Since $\mathcal H_t>0$, the distribution is not degenerate on $o_t$;
therefore $\mathcal T_t>0$, and the division by $\mathcal T_t$
is well-defined.

\noindent Using the expression derived above,
\[
\frac{\partial \mathcal T_t}{\partial l_{t,k}}
=
-\mathcal T_t p_{t,k}
-
\frac{p_{t,k}}{\mathcal T_t}
\left(
\delta_{nk}-p_{t,k}+p_{t,n}-1
\right),
\]
we obtain
\begin{align*} 
& \sum_{k=1}^V
\frac{\partial \mathcal T_t}{\partial l_{t,k}}
\Delta l_{t,k} 
\\&=
\sum_{k=1}^V
\left[
-\mathcal T_t p_{t,k}
-
\frac{p_{t,k}}{\mathcal T_t}
\left(
\delta_{nk}^t-p_{t,k}+p_{t,n}-1
\right)
\right]
\Delta l_{t,k} \\
&=
-\mathcal T_t
\sum_{k=1}^V p_{t,k}\Delta l_{t,k}
-
\frac{1}{\mathcal T_t}
\sum_{k=1}^V
p_{t,k}
\left(
\delta_{nk}^t-p_{t,k}+p_{t,n}-1
\right)
\Delta l_{t,k}.
\end{align*}
The second summation can be simplified as follows. For $k=n$,
\[
\delta_{nk}^t-p_{t,k}+p_{t,n}-1=0.
\]
For $k\neq n$,
\[
\delta_{nk}-p_{t,k}+p_{t,n}-1
=
-\left(p_{t,k}-p_{t,n}+1\right).
\]
Therefore,
\begin{align*}
& \sum_{k=1}^V
\frac{\partial \mathcal T_t}{\partial l_{t,k}}
\Delta l_{t,k} \\ &=
-\mathcal T_t
\sum_{k=1}^V p_{t,k}\Delta l_{t,k}
+
\frac{1}{\mathcal T_t}
\sum_{k\neq n}
p_{t,k}
\left(
p_{t,k}-p_{t,n}+1
\right)
\Delta l_{t,k}.
\end{align*}
So, in general, we have:
\begin{align*}
& \frac{d(\log \mathcal{T}_{t})}{d(\log \mathcal{H}_{t})}
= \frac{\mathcal{H}_{t}}{\mathcal{T}_{t}}
\frac{\sum_{k=1}^{V} \frac{\partial \mathcal{T}_{t}}{\partial l_{t,k}} \Delta l_{t,k}}
{\sum_{k=1}^{V} \frac{\partial \mathcal{H}_{t}}{\partial l_{t,k}} \Delta l_{t,k}}  \\
&= \frac{\mathcal{H}_{t}}{\mathcal{T}_{t}}
\frac{
\begin{aligned}
&\mathcal{T}_{t}\sum_{k=1}^{V} p_{t,k} \Delta l_{t,k}
-\frac{1}{\mathcal{T}_{t}}\sum_{k\neq n} p_{t,k}\\
&\quad \cdot (p_{t,k}-p_{t,n}+1) \Delta l_{t,k}
\end{aligned}}
{\sum_{k=1}^{V} p_{t,k} (\log p_{t,k} + \mathcal{H}_{t})\Delta l_{t,k}} \\
&=
\frac{
\begin{aligned}
&\sum_{k=1}^{V} p_{t,k} \Delta l_{t,k}
-\frac{1}{\mathcal{T}_{t}^2}\sum_{k\neq n} p_{t,k}\\
&\quad \cdot (p_{t,k}-p_{t,n}+1) \Delta l_{t,k}
\end{aligned}}
{\begin{aligned}
&\sum_{k=1}^{V} p_{t,k} \Delta l_{t,k}
+\frac{1}{\mathcal{H}_{t}}\sum_{k=1}^{V} p_{t,k}\\
&\quad \cdot (\log p_{t,k})\Delta l_{t,k}
\end{aligned}}.
\end{align*}
According to our assumption, if we only consider a small amount on $\Delta l_{t,n}$, then we obtain:
\begin{align*}
    \frac{d(\log \mathcal{T}_{t})}{d(\log \mathcal{H}_{t})} & = \frac{\mathcal{H}_{t}}{\mathcal{T}_{t}} \frac{\mathcal{T}_{t}p_{t,n}\Delta l_{t,n}}{p_{t,n}(\log p_{t,n}+\mathcal{H}_{t})\Delta l_{t,n}}=\frac{\mathcal{H}_{t}}{\mathcal{H}_{t}+\log p_{t,n}} \\ & =\frac{1}{1+\frac{\log \pi_{\theta}(o_t|\bd{q}, \bd{o}_{<t})}{\mathcal{H}_{t}}} = \frac{1}{\text{RSI}_t}.
\end{align*}
\end{proof}

\end{document}